%% file: main.tex
\documentclass[sigconf,nonacm]{acmart}
\AtBeginDocument{%
  }
    
\usepackage{subcaption}
\usepackage{xurl}
\usepackage{enumitem}
\usepackage{amsmath}
\usepackage{multirow}
\usepackage{booktabs}
\usepackage{xcolor}
\definecolor{posgreen}{RGB}{0,120,60}
\definecolor{negred}{RGB}{170,60,60}

\setcopyright{acmlicensed}
\copyrightyear{2018}
\acmYear{2018}
\acmDOI{XXXXXXX.XXXXXXX}
\acmConference[Conference acronym 'XX]{Make sure to enter the correct
  conference title from your rights confirmation email}{June 03--05,
  2018}{Woodstock, NY}
\acmISBN{978-1-4503-XXXX-X/2018/06}

\begin{document}

\title{Less can be More: Relieving RAG Bottlenecks\\ via Evidence Frontloading and Pressure-Adaptive Budgeting}

\author{Weibin Cai}
\email{weibin44@data.syr.edu}
\affiliation{%
  \institution{Data Lab, EECS Department}
  \institution{Syracuse University}
  \city{Syracuse}
  \state{NY}
  \country{USA}
}

\author{Reza Zafarani}
\email{reza@data.syr.edu}
\affiliation{%
  \institution{Data Lab, EECS Department}
  \institution{Syracuse University}
  \city{Syracuse}
  \state{NY}
  \country{USA}
}

\renewcommand{\shortauthors}{Cai and Zafarani}


\begin{abstract}
Existing methods for improving Retrieval-Augmented Generation (RAG) efficiency mainly optimize downstream LLM generation, such as context compression or serving optimization. 
However, RAG is an end-to-end system, and its bottleneck can shift between upstream reranking and downstream generation under different serving loads and reranking budgets.
In this paper, we first empirically characterize this shifting-bottleneck behavior and show that upstream reranking can become the dominant bottleneck under high query rates or large reranking budgets. 
Reducing the reranking budget can relieve this bottleneck, but it may also drop supporting evidence and degrade recall. 
To address this problem, we propose \textbf{\textsf{PACE}} (\textbf{P}rioritized \textbf{A}daptive \textbf{C}overage of \textbf{E}vidence), a training-free framework that combines \textit{evidence frontloading} with \textit{pressure-adaptive budgeting}. 
\textsf{PACE} first reorders candidates by marginal evidence coverage, prioritizing documents that are query-relevant, complementary, and useful for forming multi-hop evidence chains. 
We show that this objective is monotone submodular, giving greedy selection a $(1-1/e)$ approximation guarantee. 
\textsf{PACE} then dynamically adjusts the reranking budget according to the relative pressure of the reranker and the LLM. 
Experiments on three multi-hop QA datasets and online serving simulations show that \textsf{PACE} improves evidence recall, reduces p95 latency under ranking-heavy workloads. 
More importantly, the two components together reveal that \textit{less can be more}: an evidence-dense top-ranked candidates enable higher final recall with fewer reranked documents. 
% \footnote{Some core anonymous code is available at \url{https://anonymous.4open.science/r/XXXX-XXXX/}; we will open-source all the organized code and data upon acceptance.} 
\end{abstract}




\maketitle

\input{Sections/Introduction}

\input{Sections/Preliminary_Analysis}

\input{Sections/Method}

\input{Sections/Experiments}

\input{Sections/Related_Work}

\vspace{-2mm}
\section{Conclusion}

In this paper, we study RAG as an end-to-end serving system and show that its bottleneck is not fixed at LLM generation. 
Instead, the dominant bottleneck can shift between reranking and generation as query arrival rates and reranking budgets change. 
Motivated by this observation, we propose \textsf{PACE}, a training-free framework that relieves upstream reranking bottlenecks while preserving evidence recall. 
\textsf{PACE} frontloads useful evidence into the top of the candidate ranking through a monotone submodular marginal coverage objective, and then adaptively selects the reranking budget based on real-time reranker and LLM pressure. 
Experiments on multi-hop QA datasets and online serving simulations demonstrate that \textsf{PACE} improves evidence recall, substantially reduces latency under ranking-heavy workloads, and shows that less can be more when the top-ranked candidates are evidence-dense: a smaller reranking budget can lead to higher final evidence recall.

\bibliographystyle{ACM-Reference-Format}
\bibliography{reference}

\appendix

\end{document}

%% file: Sections/Introduction.tex
\section{Introduction}

\begin{figure}
    \centering
    \includegraphics[width=\linewidth]{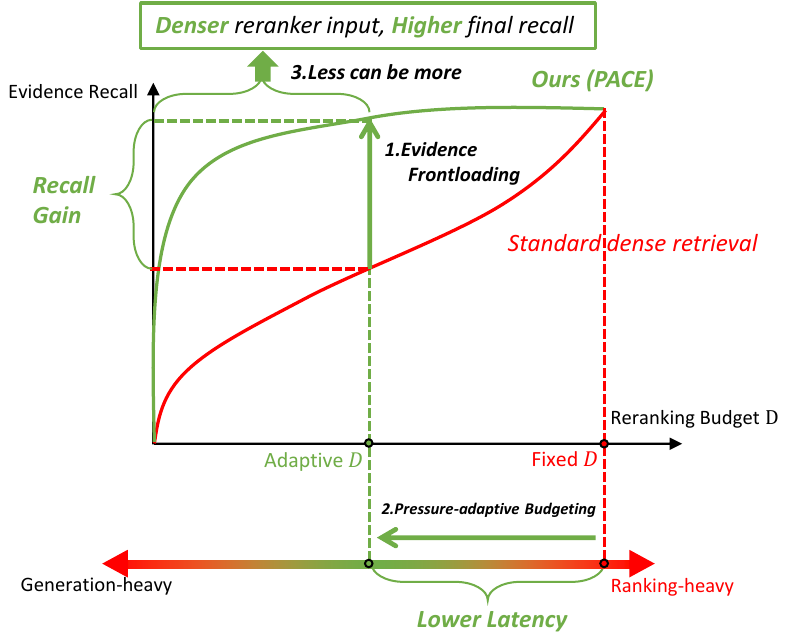}
    \vspace{-7.5mm}
    \caption{
    How \textsf{PACE} improves RAG effectiveness and efficiency.  
    (1) Evidence frontloading makes top-ranked candidates more evidence-dense, achieving higher evidence recall at a smaller reranking budget $D$. 
    (2) Pressure-adaptive budgeting then reduces $D$ under ranking-heavy workloads to lower latency. 
    (3) Together, \textsf{PACE} enables denser reranker input and higher final recall with fewer reranked documents. 
    }
    \label{fig:intro}
    \vspace{-6mm}
\end{figure}

RAG has become a widely used paradigm for equipping Large Language Models (LLMs) with external knowledge across diverse downstream tasks~\cite{shuster2021retrieval,lewis2020retrieval,borgeaud2022improving}. 
A typical RAG pipeline first retrieves query-relevant chunks from a corpus using a dense retriever~\cite{karpukhin2020dense,xiong2020approximate,khattab2020colbert}, then applies a reranker~\cite{nogueira2019passage,sun2023chatgpt,nogueira2020document} to refine their order and feeds the top-$k$ chunks to the LLM as context. 
The quality of the generated answer therefore largely depends on whether the context (top-$k$ chunks) covers the evidence needed to answer the query. 
Increasing $k$ can improve evidence recall, but longer contexts also make RAG systems harder to process effectively and efficiently: useful information may be lost in the middle~\cite{liu2024lost}, noisy sentences can degrade generation~\cite{shi2023large}, and long inputs introduce substantial inference overhead~\cite{dejean2026efficient,yu2024rankrag}. 
Existing work addresses these issues from two main directions. 
At the system level, techniques such as KV-cache reuse~\cite{zheng2024sglang}, scheduling~\cite{yu2022orca}, and optimized prefilling~\cite{agrawal2024taming} accelerate LLM inference. 
At the data and algorithm level, context compressors~\cite{chirkova2025provence,xu2023recomp,jiang2023llmlingua} aim to preserve core information while removing irrelevant tokens from the retrieved context. 
These methods reduce input length or accelerate generation, thereby alleviating downstream bottlenecks (i.e., generation stage bottleneck) while largely maintaining question-answering performance.

However, RAG is an end-to-end system, and improving it requires coordinating upstream retrieval/reranking with downstream generation for both effectiveness and efficiency. 
For effectiveness, if necessary evidence is not recalled upstream, removing noisy context downstream cannot recover the missing information or improve answer accuracy. 
For efficiency, If queries are already queued at the reranker, optimizing only the generator may not improve end-to-end latency. 
Consequently, approaches that mainly optimize downstream generation may not generalize across different RAG configurations and serving loads. 
In particular, as the reranking budget (i.e., the number of retrieved candidates sent to the reranker) and QPS (queries per second) vary, the system bottleneck can shift between ranking and generation. 
When queries arrive frequently and many documents must be reranked, the ranking stage can become the bottleneck; otherwise, latency may be dominated by generation. 
We refer to these two scenarios as \textit{ranking-heavy} and \textit{generation-heavy}.
Although using a small fixed reranking budget can relieve ranking-heavy workloads, this fixed budget is difficult to choose under dynamic serving loads, and an overly small budget can miss necessary evidence and degrade answer quality. 
Therefore, how to \textit{dynamically relieve upstream ranking bottlenecks while preserving evidence recall} remains an open problem. A promising solution requires two key properties, as illustrated in Figure~\ref{fig:intro}: 

% \vspace{0.3mm}
\noindent \textbf{Evidence frontloading.} 
This is especially important for multi-hop questions, where answering a query often requires multiple supporting documents rather than a single highly relevant chunk. 
If these supporting documents can be frontloaded into the first few candidates, the system can use a smaller reranking budget, reducing upstream workload while preserving evidence recall. 
Existing retrieval-side methods partially address this goal through multi-hop retrieval or diversity-aware ranking. 
Multi-hop retrievers usually acquire supporting evidence through iterative retrieval or reasoning-guided search~\cite{xiong2020answering,khattab2021baleen,trivedi2023interleaving}, but they require specialized retriever training or introduce additional LLM-retrieval iterations, which can increase serving cost. 
In contrast, the goal in this scenario is \textit{not to acquire new evidence through extra retrieval steps, but to prioritize evidence within an existing candidate pool} so that a smaller reranking budget can still preserve recall under serving pressure. 
Diversity-aware methods are closer to this setting because they usually reorder existing candidates by reducing redundancy. 
For example, maximal marginal relevance (MMR)~\cite{carbonell1998use} penalizes similarity to previously selected documents, while Dartboard~\cite{pickett2024better} optimizes relevant information gain for diverse RAG retrieval. 
However, these methods remain limited for multi-hop evidence frontloading due to three issues: 
(i) \textit{coarse-grained relevance modeling}, as they typically rely on document-level similarity rather than fine-grained semantic dimensions; 
(ii) \textit{missing inter-document dependencies}, as a supporting document may be weakly related to the query but strongly connected to another evidence document; and 
(iii) \textit{limited robustness}, as their performance often depends on dataset-sensitive hyperparameters and lacks principled performance guaranties. 

% \vspace{0.3mm}
\noindent \textbf{Pressure-adaptive budgeting.}  
Instead of relying on a fixed budget, it should dynamically choose the largest budget that does not make reranking a bottleneck relative to generation. 
This requires jointly considering upstream reranking pressure and downstream generation pressure, so that the system can preserve as much evidence as possible while relieving ranking-heavy workloads. 
Together, these two properties enable the system to use \textit{a smaller reranking budget with higher evidence recall}, thereby reducing upstream bottlenecks while maintaining or even improving evidence recall.

\noindent \textbf{In this work.} 
We first characterize RAG bottlenecks under online serving simulation with varying QPS and reranking budgets. 
Our analysis shows that the bottleneck is not determined solely by the relative size of the reranker and the LLM: \textit{increasing QPS or the reranking budget can shift the dominant bottleneck from generation to upstream reranking}. 
This finding motivates us to focus on the ranking-heavy scenario, where reducing downstream generation cost alone is insufficient. 
To address this problem, we propose \textbf{\textsf{PACE}}, short for \textit{\textbf{P}rioritized \textbf{A}daptive \textbf{C}overage of \textbf{E}vidence}, to relieve upstream reranking bottlenecks while preserving evidence recall. 
\textsf{PACE} consists of two components. 
First, \textbf{\textit{evidence frontloading}} reorders candidates so that the top-ranked candidates contain evidence that is directly relevant to the query, complementary to already selected documents, and useful for forming a complete evidence chain. 
We formulate this as a marginal evidence coverage problem: using fewer documents to cover more query semantics. 
Each document's coverage is weighted by its direct query relevance and its potential to bridge evidence chains. 
This objective is monotone submodular, giving greedy frontloading a $(1-1/e)$ approximation guarantee. 
Second, \textbf{\textit{pressure-adaptive budgeting}} dynamically adjusts the reranking budget according to the relative queueing pressure of the reranker and the LLM, reducing reranker workload when upstream ranking becomes the bottleneck. 
Together, the two components enable a key effect: \textit{less can be more when the top-ranked candidates are evidence-dense}. 
Although adaptive budgeting sends fewer candidates to the reranker, evidence frontloading makes these candidates more useful and less noisy, increasing the chance that the reranker promotes complete evidence into the final top-$K$ context. 
Figure~\ref{fig:intro} summarizes how these two components jointly improve RAG effectiveness and efficiency, leading to the ``less can be more'' effect. 
Our contributions are as follows: 
\begin{itemize}
    \item We empirically identify shifting bottlenecks in RAG serving, showing that the dominant bottleneck can move between reranking and generation as the number of queries per second and the reranking budget change. 
    \item We propose \textsf{PACE}, a training-free framework that combines marginal evidence frontloading with pressure-adaptive budgeting to reduce upstream reranking workload while preserving evidence recall. We further show that the proposed evidence coverage objective of \textbf{\textsf{PACE}} is monotone submodular, which leads to its greedy selection to provide a $(1-1/e)$ approximation guarantee under a cardinality constraint. 
    \item Experiments on three multi-hop QA datasets and online serving simulations show that \textsf{PACE} improves evidence recall, substantially reduces latency under ranking-heavy workloads, and demonstrate that \textit{less can be more when top-ranked candidates are evidence-dense}: achieve higher final recall with a smaller reranking budget. 
\end{itemize}
% Some core anonymous code is available at \url{https://anonymous.4open.science/r/PACE-RAG-3CC5/}; we will open-source all the organized code and data upon acceptance. 

%% file: Sections/Preliminary_Analysis.tex
\section{RAG Bottlenecks Shift Across Configurations and Loads}
\label{sec:shift_bottleneck}

\begin{figure*}[t]
    \centering
    \begin{subfigure}[t]{0.32\textwidth}
        \centering
        \includegraphics[width=\linewidth]{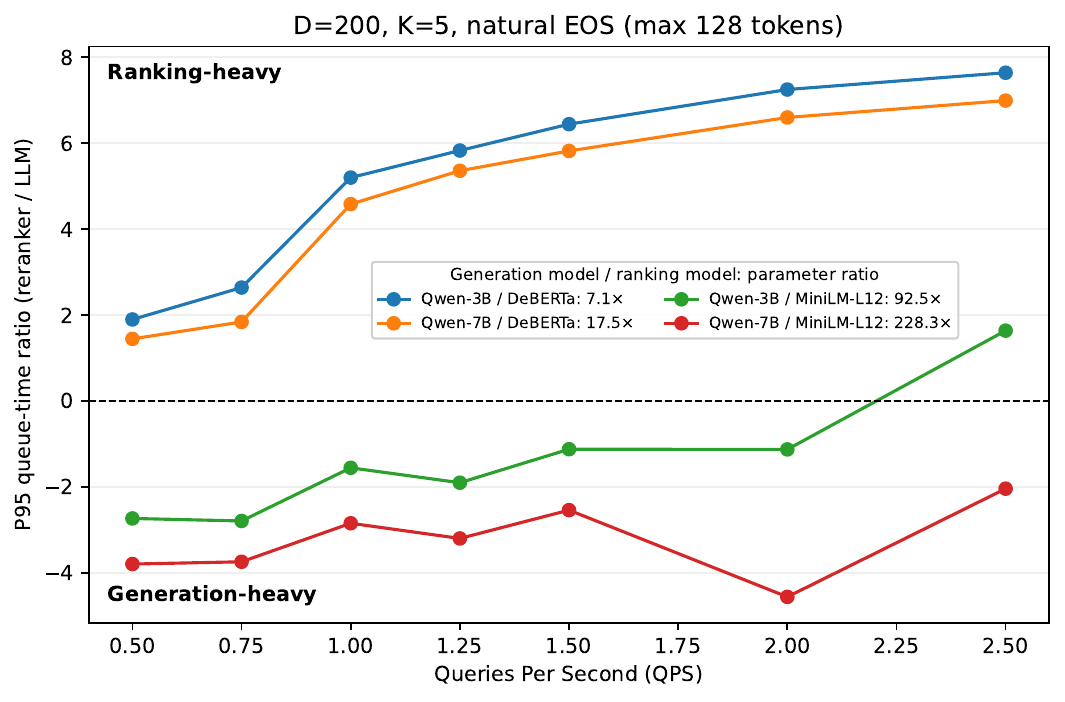}
        \caption{Bottleneck tendency across reranker--LLM model pairs under increasing QPS.}
        \label{fig:shift_across_model}
    \end{subfigure}
    \hfill
    \begin{subfigure}[t]{0.32\textwidth}
        \centering
        \includegraphics[width=\linewidth]{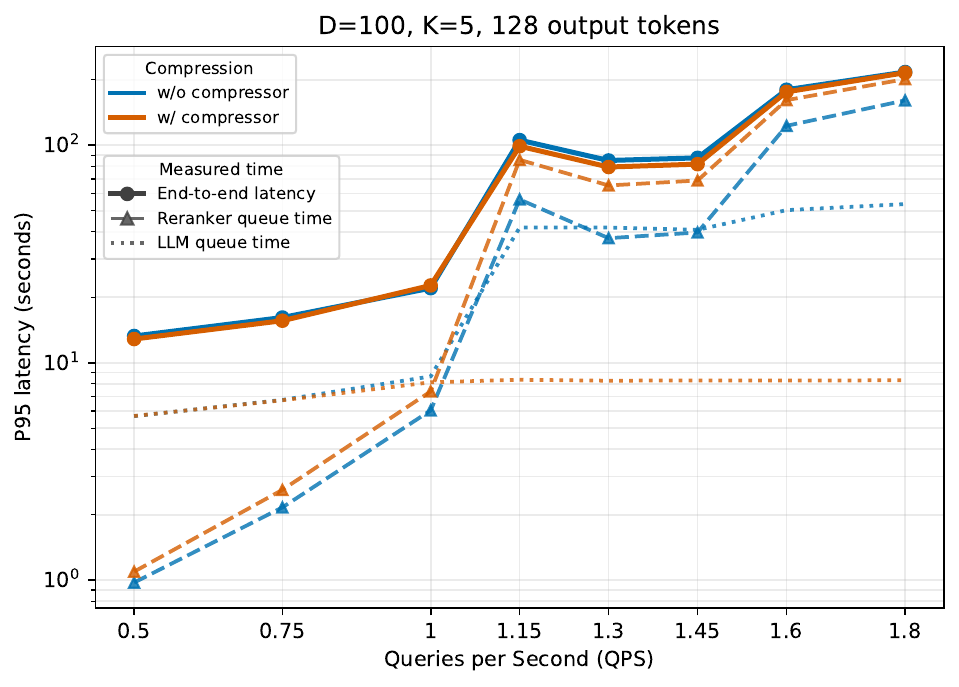}
        \caption{Ranking-heavy workloads}
        \label{fig:ranking_heavy}
    \end{subfigure}
    \hfill
    \begin{subfigure}[t]{0.32\textwidth}
        \centering
        \includegraphics[width=\linewidth]{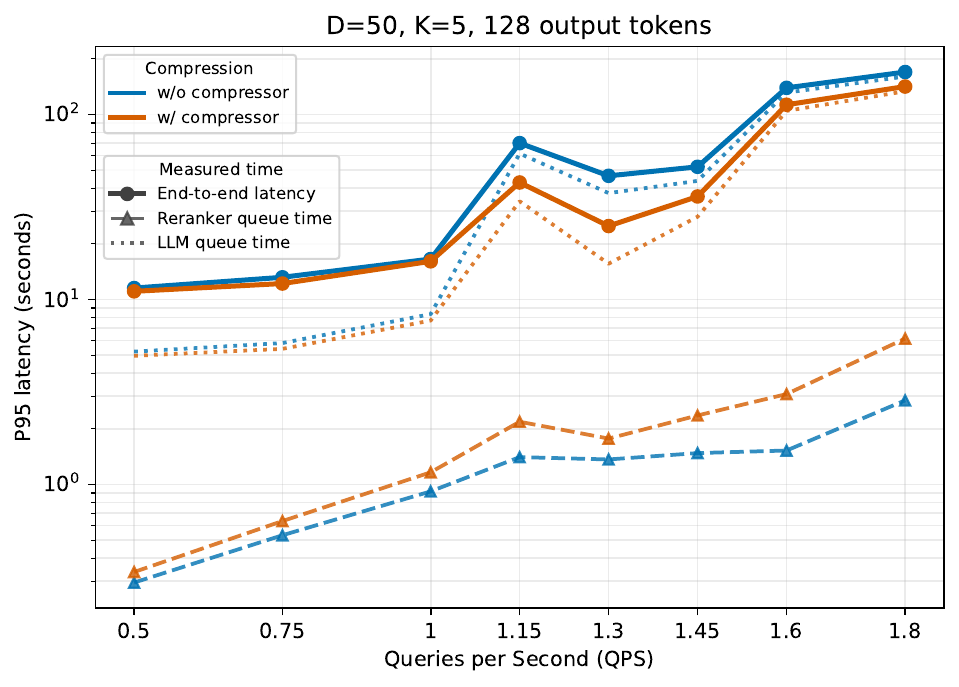}
        \caption{Generation-heavy workloads}
        \label{fig:generation_heavy}
    \end{subfigure}
    \vspace{-2mm}
    \caption{Shifting RAG bottlenecks under different configurations and serving loads. Here, $D$ denotes reranking budget and $K$ denotes the number of documents sent to the LLM. In (a), we vary QPS across two rerankers: \texttt{trecdl22-crossencoder-debertav3}, and \texttt{ms-marco-MiniLM-L12-v2}, and two LLMs: \texttt{Qwen2.5-3B-Instruct}, and \texttt{Qwen2.5-7B-Instruct}. We report $\log_2\left( \frac{\text{P95 reranker queue time}}{\text{P95 LLM queue time}} \right)$, where positive values indicate ranking-heavy serving and negative values indicate generation-heavy serving. In (b) and (c) we use \texttt{Qwen2.5-3B-Instruct} and \texttt{trecdl22-crossencoder-debertav3} pair and vary $D$ with and without context compressor.}
    \label{fig:shifting_bottleneck}
    \vspace{-4mm}
\end{figure*}

% \begin{table}[t]
% \centering
% \caption{Notation.}
% \label{tab:notation}
% \begin{tabular}{cl}
% \toprule
% Symbol & Meaning \\
% \midrule
% $D$ & \# reranked documents (i.e., reranking budget) \\
% $K$ & \# top documents sent to LLM \\
% $I$ & \# LLM input tokens \\
% $O$ & \# LLM output tokens \\
% \bottomrule
% \end{tabular}
% \end{table}

In a RAG system, the LLM contains most of the parameters and computational complexity, so it is often assumed to be the main serving bottleneck. 
However, the bottleneck is not always fixed at generation. 
It can shift between upstream reranking and downstream generation under different model choices, configurations (e.g., reranking budget $D$), and serving loads (e.g., queries per second). 
We characterize this shift by comparing how long requests wait in the reranker queue and the LLM queue: 
if requests wait longer in the reranker queue than in the LLM queue, reranking is the bottleneck; otherwise, generation is the bottleneck. 
We report p95 queueing latency, i.e., the latency below which 95\% of requests fall, because serving bottlenecks often appear as tail latency and indicate system saturation under high load. 
This paper focuses on the reranking and generation stages, and measures their queueing delays to identify which stage limits system throughput. 
To study RAG bottlenecks under realistic serving conditions, we use an open-loop workload, where requests are issued according to a fixed arrival process independent of previous request completion. 
The detailed datasets, models, and serving setup are described in Section~\ref{sec:exp_set}; here we focus on the observed bottleneck behavior.

\vspace{1mm}
\noindent \textbf{\textit{1. Model size affects bottleneck tendency, while load and reranking budget determine the actual bottleneck.}} 
Figure~\ref{fig:shift_across_model} shows that, across different reranker--LLM model pairs, increasing queries per second generally increases upstream reranking pressure. 
Relative model size affects the bottleneck, but high load can still make reranking the bottleneck even with a much larger LLM. 
For example, the parameter ratio between Qwen-3B and MiniLM-L12 is 92.5, and this pair behaves as generation-heavy under low QPS but shifts to ranking-heavy when QPS reaches 2.5. 
Reranking budget also changes the bottleneck. 
For the Qwen-3B and DeBERTa pair, Figure~\ref{fig:ranking_heavy} shows that with reranking budget $D=100$, end-to-end latency becomes dominated by reranker queueing delay once QPS reaches 1, whereas Figure~\ref{fig:generation_heavy} shows that with $D=50$, latency remains dominated by LLM queueing delay.

\vspace{1mm}
\noindent \textbf{\textit{2. Relieving downstream bottlenecks cannot necessarily reduce upstream reranking pressure.}} 
Figures~\ref{fig:ranking_heavy} and~\ref{fig:generation_heavy} show that context compression reduces LLM queueing delay by shortening the input context. 
As a result, it improves end-to-end latency when generation is the dominant bottleneck. 
However, this gain becomes limited when the system is ranking-heavy. 
In Figure~\ref{fig:ranking_heavy}, once QPS reaches 1, most latency comes from the reranker queue. 
Although compression accelerates LLM processing, it does not reduce reranker queueing delay; moreover, the compressor consumes additional compute, which can further increase reranker pressure and queueing time. 
As a result, its end-to-end benefit remains limited, especially under high QPS and large $D$.

%% file: Sections/Method.tex
\section{\textsf{PACE}: Prioritized Adaptive Coverage of Evidence}
The observations in Section~\ref{sec:shift_bottleneck} show that RAG bottlenecks shift with serving load and system configuration. 
To relieve shifting bottlenecks, the system should adaptively allocate the reranking budget $D$ rather than using a fixed value. 
However, simply reducing $D$ may miss necessary evidence and degrade answer quality. 
We therefore propose \textsf{PACE} (Prioritized Adaptive Coverage of Evidence), to pace the RAG system by addressing two requirements: 
(1) adapting the reranking budget to system pressure, and 
(2) preserving evidence recall by moving useful evidence into top-ranked candidates.

\subsection{Evidence Frontloading}
\label{sec:evidence_frontload}
To preserve evidence recall under a small reranking budget, we first frontload useful evidence to the top of the candidate ranking. 
Given a query $q$ and a candidate document set $\mathcal{C}=\{d_i\}_{i=1}^{D_{\max}}$ returned by the dense retriever, where $D_{\max}$ denotes the maximum of reranking budget, e.g., top-$100$, our goal is to reorder $\mathcal{C}$ so that the first few documents cover as much necessary and non-redundant evidence as possible. 
To make the top-ranked documents useful, each selected document should satisfy two conditions. 
First, it should be relevant to the query and cover part of the query's information need, such as a specific entity, event, or location. 
Second, it should provide complementary evidence rather than repeat dimensions that have already been covered by previously selected documents. 
Thus, documents should be prioritized not only by their individual query relevance, but also by their marginal contribution to the uncovered evidence space. 
As more documents are selected, the remaining uncovered dimensions become fewer, and the gain of adding redundant documents naturally decreases. 
We formalize this intuition as a marginal evidence coverage problem.

\paragraph{Marginal evidence covergae.} 
Let $q(v)\ge 0$ and $d_i(v)\ge 0$ denote the representation values of query $q$ and document $d_i$ on semantic dimension $v$, respectively. 
The value on each dimension reflects the amount of information carried by the query or document along that dimension. 
For a selected document set $S\subseteq \mathcal{C}$, we define the coverage objective as 
\begin{equation}
    F(S) = \sum_v q(v) \max_{d_i \in S}\left[ \sqrt{\rho_i} d_i(v) \right]. 
    \label{eq:coverage}
\end{equation}
where $\rho_i$ denotes the query-document relevance weight, which can be instantiated by the dense retriever score. 
The current coverage of dimension $v$ is 
\begin{equation}
    m_S(v) = \max_{d_i \in S}\left[ \sqrt{\rho_i} d_i(v) \right]. 
\end{equation}
Thus, the marginal gain of adding document $d_i$ to $S$ is 
\begin{equation}
    \Delta(d_i \mid S) = \sum_v q(v) \max \left[ \sqrt{\rho_i} d_i(v) - m_S(v), 0 \right]. 
\end{equation}
At each step, we greedily select the document with the largest marginal gain: 
\begin{equation}
    d^\star = \arg\max_{d_i \in \mathcal{C} \setminus S} \Delta(d_i \mid S). 
\end{equation}
This produces a ranking prefix that prioritizes relevant and complementary evidence.

\paragraph{Soft-anchor relevance refinement.} 
However, using only the dense retriever score as $\rho_i$ can miss supporting documents that are weakly related to the query but strongly connected to other evidence documents. 
This is common in multi-hop QA, where one document may provide an intermediate fact or entity bridge rather than directly matching the query. 
To capture such dependencies, we refine the relevance weight using soft anchors, i.e., candidate documents with high query relevance that serve as evidence seeds. 
Documents similar to these anchors receive higher weights, allowing the ranking prefix to include not only directly query-relevant documents but also bridge documents that connect multiple supporting evidence into a more complete evidence chain. 
Together with \textit{marginal evidence coverage}, this refined relevance weight enables the system to frontload evidence that is both comprehensive and useful. 
We next describe how to compute this refinement. 

Let $\rho_i$ denote the initial query-document relevance score from the dense retriever, and let $\operatorname{sim}(d_i,d_j)$ denote the similarity between two candidate documents. 
We first robustly standardize the relevance scores using the median absolute deviation: 
\begin{equation}
    z_i = \frac{\rho_i - \operatorname{median}(\rho)}{\operatorname{MAD}(\rho)+\epsilon},
    \label{eq:z_stand}
\end{equation}
where $\epsilon=10^{-12}$ is used for numerical stability, and 
\begin{equation}
    \operatorname{MAD}(\rho) = \operatorname{median}_i\left| \rho_i - \operatorname{median}(\rho) \right|.
\end{equation}
We then convert the standardized scores into soft-anchor weights:
\begin{equation}
    w_i = \frac{\exp(z_i)}{\sum_j \exp(z_j)}. 
\end{equation}
Thus, candidates with higher query relevance receive larger weights and act as soft evidence anchors. 
We estimate each document's soft-anchor relevance by aggregating its similarity to these anchors: 
\begin{equation}
    a_j = \frac{\sum_{i\neq j} w_i \operatorname{sim}(d_i, d_j)}{\sum_{i\neq j}w_i}. 
\end{equation}
Finally, we combine the direct query-document relevance and the anchor-based relevance as
\begin{equation}
    \tilde{\rho}_j = 1 - (1 - b_j)(1 - \bar{a}_j), 
    \label{eq:anchor_refine}
\end{equation}
where $b_j=\operatorname{norm}(\rho_j)$, $\bar{a}_j=\operatorname{norm}(a_j)$, and $\operatorname{norm}(\cdot)$ denotes min-max normalization.  
The refined weight $\tilde{\rho}_j$ is then used as the relevance weight in Eq.~\ref{eq:coverage}, allowing the coverage objective to favor documents that are directly relevant to the query while also recovering documents connected to query-relevant evidence through document-document dependencies. 
This refinement can be applied either before or after reranking to improve recall; the only difference is the source of the query-document relevance score in Eq.~\ref{eq:z_stand}. 
Before reranking, $\rho_i$ is obtained from the dense retriever, while after reranking, it is obtained from the reranker. 
We evaluate this combination of usage in Section~\ref{sec:pace_improve}, Figure~\ref{fig:online_recall}.

\paragraph{Theoretical property.} 
Beyond its empirical effectiveness, the proposed \textit{marginal evidence coverage} objective $F(S)$ is monotone submodular, which gives greedy selection a standard approximation guarantee to the optimal evidence coverage under a cardinality constraint.

\begin{theorem}
Assume $q(v)\ge 0$, $\tilde{\rho}_i \ge 0$, and $d_i(v)\ge 0$ for all $v$ and $d_i$. 
Then $F(S)$ is monotone submodular. 
\end{theorem}

\begin{proof}
First, we prove monotonicity. 
For any $A \subseteq B \subseteq \mathcal{C}$ and any dimension $v$, we have 
\begin{equation}
    \max_{d_i \in A} \left[ \sqrt{\tilde{\rho}_i} d_i(v) \right] \le \max_{d_i \in B} \left[ \sqrt{\tilde{\rho}_i} d_i(v) \right]. 
\end{equation}
Since $q(v)\ge 0$, summing over all dimensions gives 
\begin{equation}
    F(A) \le F(B). 
\end{equation}
Therefore, $F(\cdot)$ is monotone. 

Next, we prove submodularity by showing diminishing returns. 
For any $A \subseteq B \subseteq \mathcal{C}$ and any $d_j \notin B$, define 
\begin{equation}
    m_A(v) = \max_{d_i \in A} \left[ \sqrt{\tilde{\rho}_i} d_i(v) \right], \quad m_B(v) = \max_{d_i \in B} \left[ \sqrt{\tilde{\rho}_i} d_i(v) \right]. 
\end{equation}
Because $A \subseteq B$, we have $m_A(v) \le m_B(v)$ for every $v$. 
Therefore, 
\begin{equation}
    \max \left[ \sqrt{\tilde{\rho}_j}d_j(v) - m_A(v), 0 \right] \ge \max \left[ \sqrt{\tilde{\rho}_j}d_j(v) - m_B(v), 0 \right]. 
\end{equation}
Multiplying by $q(v)\ge 0$ and summing over $v$ yields 
\begin{equation}
    \Delta(d_j \mid A) \ge \Delta(d_j \mid B). 
\end{equation}
Thus, $F(\cdot)$ is monotone submodular. 

\end{proof}

\begin{corollary}
Under a cardinality constraint $|S| \le K$, let $S_{\mathrm{greedy}}$ be the set selected by the greedy algorithm that iteratively adds the document with the largest marginal gain. 
Then, 
\begin{equation}
    F(S_{\mathrm{greedy}}) \ge (1-1/e)F(S^\star), 
\end{equation}
where
\begin{equation}
    S^\star = \arg\max_{|S| \le K}F(S). 
\end{equation}
is the optimal document set under the same constraint. 
\end{corollary}

\begin{proof}
Since $F(S)$ is non-negative, monotone, and submodular, the result follows from the classical greedy approximation guarantee for monotone submodular maximization under a cardinality constraint~\cite{nemhauser1978analysis}. 
\end{proof}
This result shows that greedy evidence frontloading achieves at least a $(1-1/e)$ approximation to the optimal evidence coverage, while remaining efficient enough for reranking-time use.

\subsection{Pressure-adaptive Budgeting}

\begin{table}[t]
\centering
\small
\caption{Notation for pressure-adaptive budgeting.}
\vspace{-2mm}
\label{tab:pressure_notation}
\begin{tabular}{ll}
\toprule
Symbol & Meaning \\
\midrule
$P_R$ & Number of unprocessed document pairs in the reranker queue \\
$\hat{\mu}_R$ & Real-time estimated reranker throughput \\
$W_R$ & Estimated time to clear the reranker queue \\
$|Q_L|$ & Number of queries waiting in the LLM queue \\
$B_L$ & LLM batch size \\
$B_R$ & Reranker batch size \\
$\hat{T}_L$ & Estimated execution time of one LLM batch \\
$T_{\mathrm{remain}}$ & Remaining time of the current LLM batch \\
$E_R$ & Excess reranker workload compared with the LLM \\
$D_{\min}$ & Minimum reranking budget \\
$D_{\max}$ & Maximum reranking budget \\
$t_{\mathrm{now}}$ & Current timestamp \\
$t_{\mathrm{start}}$ & Start timestamp of the current LLM batch \\
\bottomrule
\end{tabular}
\vspace{-5mm}
\end{table}

Evidence frontloading makes the ranking prefix more evidence-dense, but the system still needs to decide how many candidates should be processed by the reranker under changing serving pressure. 
A fixed reranking budget $D$ is suboptimal: a large $D$ preserves recall but can create a reranker bottleneck, while a small $D$ reduces latency but may miss necessary evidence. 
We therefore introduce pressure-adaptive budgeting, which dynamically selects the largest affordable reranking budget for each query based on the relative pressure of the reranker and the LLM. 
Together with evidence frontloading, this policy balances the pressure between the two stages while preserving high evidence recall. 
Table~\ref{tab:pressure_notation} summarizes the notation used in this section.

We estimate reranker pressure as the time needed to clear the current reranker backlog: 
\begin{equation}
    W_R = \frac{P_R}{\hat{\mu}_R}. 
\end{equation}
For the LLM, we estimate the remaining service time by combining the current batch and the queued batches: 
\begin{equation}
    T_{\mathrm{remain}} = \max \left( \hat{T}_L - (t_{\mathrm{now}} - t_{\mathrm{start}}), 0 \right), 
\end{equation}
\begin{equation}
    W_L = T_{\mathrm{remain}} + \left \lceil \frac{|Q_L|}{B_L} \right \rceil \hat{T}_L. 
\end{equation}
When $W_R \le W_L$, reranking is not more congested than generation, so the system uses $D_{\max}$. 
When $W_R > W_L$, we estimate the excess reranker backlog as 
\begin{equation}
    E_R = \hat{\mu}_R (W_R - W_L). 
\end{equation}
Given reranker batch size $B_R$, we reduce the reranking budget by full reranker batches: 
\begin{equation}
    D(q) = \max \left ( D_{\min}, D_{\max} - B_R \left\lfloor \frac{E_R}{B_R} \right\rfloor \right). 
\end{equation}
This rule keeps the largest possible reranking budget when reranking is not the bottleneck, and decreases $D(q)$ only when reranker pressure exceeds LLM pressure. 
Together with evidence frontloading, it reduces upstream workload under ranking-heavy conditions while keeping useful evidence concentrated in the processed prefix.

%% file: Sections/Experiments.tex
\section{Experiments}
We evaluate \textsf{PACE} from two perspectives: whether evidence frontloading improves evidence recall among the top-ranked candidates, and whether pressure-adaptive budgeting reduces end-to-end latency under shifting bottlenecks while preserving recall. 
All experiments are conducted on Quadro RTX 6000 GPUs with approximately 22 GB of available memory.

\subsection{Experimental Settings}
\label{sec:exp_set}

\vspace{1mm}
\noindent \textbf{Model selection.} 
Following prior work~\cite{chirkova2025provence}, we use \path{splade-v3}~\cite{lassance2024splade} as the retriever. 
It produces non-negative query and document representations, satisfying the assumptions in Section~\ref{sec:evidence_frontload}. 
We use \path{trecdl22-crossencoder-debertav3}~\cite{dejean2024thorough} as the reranker, \path{Provence}~\cite{chirkova2025provence} as the context compressor, and \path{Qwen2.5-3B-Instruct}~\cite{qwen2025qwen2} as the generator under a resource-limited serving setting. 
All generation experiments use prefilling. 

\vspace{1mm}
\noindent \textbf{Datasets.}  
We evaluate on the dev splits of three multi-hop QA datasets with gold evidence annotations: HotpotQA~\cite{yang2018hotpotqa} with 1,087 queries, MuSiQue~\cite{trivedi2022musique} with 2,317 queries, and 2WikiMultiHop\allowbreak QA~\cite{xanh2020_2wikimultihop} with 2,861 queries. 
For each dataset, we reserve an additional 100 queries to tune baselines that require hyperparameter calibration; these queries are excluded from evaluation. 
Our method does not use this calibration set. 
For HotpotQA and 2WikiMultiHopQA, we retain queries whose complete supporting evidence appears in the top-100 retrieved documents. 
This controls for initial retrieval failures and lets us focus on how effectively each method ranks supporting evidence toward earlier positions at different document budgets. 
Following Provence~\cite{chirkova2025provence}, we use the Wikipedia corpus preprocessing and passage segmentation provided through BERGEN~\cite{rau2024bergen}.
For MuSiQue, we use the official closed-context candidate paragraphs; therefore, all queries are retained, the maximum reranking budget is 20, and no external Wikipedia retrieval or additional passage segmentation is performed.

\vspace{1mm}
\noindent \textbf{Metrics.} 
To evaluate \textit{evidence frontloading}, we measure evidence recall among top-ranked candidates. 
We report \textbf{complete evidence recall@D}, the percentage of queries whose top-$D$ documents contain all gold supporting evidence, and \textbf{supporting evidence recall@D}, the average fraction of gold supporting evidence covered by the top-$D$ documents. 
We evaluate document selection for $D=1,\ldots,100$ on HotpotQA and 2WikiMultiHopQA, and $D=1,\ldots,20$ on MuSiQue, whose closed-context setting provides at most 20 candidate paragraphs per query. 
To evaluate \textit{pressure-adaptive budgeting}, we report \textbf{p95 end-to-end latency}, \textbf{p95 reranker queue time}, and \textbf{p95 LLM queue time}. 
Here, p95 latency is the latency below which 95\% of requests complete.

\vspace{1mm}
\noindent \textbf{Batch size calibration.} 
To ensure that both reranker and LLM achieve maximum processing speed on the local device, thereby yielding reliable efficiency results. 
We profile reranker and LLM on our hardware before the main experiments. 
For the reranker, batch size denotes the number of query-document pairs scored together, and we sweep $B_r \in \left\{2,4,8, 16, 32, 64\right\}$. 
For the LLM, batch size denotes the maximum number of concurrent generation sequences, and we sweep $B_l \in \left\{6,7,8,9,10,11,12\right\}$ with a maximum of 128 new tokens. 
Each candidate batch size is evaluated on the same 100 randomly sampled TydiQA queries~\cite{tydiqa}. 
In this profiling run, the reranker processes the top-50 retrieved documents and returns the top-5 documents to the LLM. 
We measure throughput, p95 latency, and GPU memory usage, repeat each run three times with a fixed random seed. 
We select the smallest batch size whose throughput is within 95\% of the best observed throughput while satisfying latency and memory constraints. 
The selected batch sizes, $B_R=8$ for DeBERTa and $B_L=10$ for Qwen2.5-3B, are fixed in all main experiments so that bottleneck shifts are caused by workload changes rather than per-setting batch-size tuning.

\vspace{1mm}
\noindent \textbf{Online serving simulation.} 
To study RAG bottlenecks under realistic serving conditions, we use an open-loop workload, where requests are issued according to a fixed arrival process independent of previous request completion. 
Queries arrive following a Poisson process~\cite{ray2025metis}, with QPS varying from 0.5 to 2.5. 
We deploy the reranker and the LLM on separate GPUs and connect them with independent asynchronous queues to enable pipeline parallelism. 
To examine whether downstream context reduction translates into system-level efficiency gains, we insert Provence~\cite{chirkova2025provence}, a context compressor, between the reranker and the LLM and run it on a separate GPU. 
For analyzing RAG bottlenecks in Figure~\ref{fig:shifting_bottleneck}, we use the first 60 seconds as warm-up and record per-stage latency over the following 300 seconds. 
For the end-to-end evaluation of \textsf{PACE} in Figure~\ref{fig:online_recall} and~\ref{fig:online_latency}, we set $D_{\min}=20$, and $D_{\max}=100$, issue all evaluation queries according to the target QPS and continue running until all queries are completed. 
This allows us to measure evidence recall under dynamically selected $D$ in an online setting. 
We fix the query order and document order across workloads for fair comparison. 
Generation uses natural end-of-sequence termination with a maximum of 128 output tokens. 
\vspace{-2mm}
\noindent \paragraph{Use of AI tools.} 
We used AI tools to assist with the implementation of experimental scripts for online serving simulation. All generated code was reviewed, tested, and validated by the authors. The AI tool was not used to generate datasets, labels or conclusions. 

% ------------------------------------------------
\vspace{1mm}
\noindent \textbf{Baselines.} 
To evaluate the effectiveness of evidence frontloading, we compare with training-free reordering methods that do not require additional model training. 
All methods operate on the same candidate set returned by the retriever. 
For baselines with tunable hyperparameters, we select the best configuration on the 100 query calibration set of each dataset and apply it unchanged to the evaluation split. 
% \vspace{-2mm}
\begin{itemize}
    \item \textbf{Standard Dense}: 
    It uses the original ranking returned by the dense retriever without any reordering. 
    \item \textbf{Rocchio Pseudo-Relevance Feedback (PRF)}~\cite{rocchio1971relevance}: 
    It assumes the top retrieved documents are pseudo-relevant, updates the query representation by combining the original query with these documents, and then reorders candidates using the expanded query.

% \begin{equation}
%     q'=\alpha\bar{q} + \frac{\beta}{M}\sum_{i \in R_M}\bar{d_i}
% \end{equation}
    
    % $q'=\alpha\bar{q} + \frac{\beta}{M}\sum_{i \in R_M}\bar{d_i}$, 
    % where $\alpha$ and $\beta$ are weights for query and document embeddings; $M$ is the number of top retrieved documents and their indices are formed as a set $R_M$. 
    % In our experiment, $\alpha=0.7$, $\beta=0.3$ and $M=5$ for HotpotQA and 2WikiMultiHopQA; $\alpha=0.3$, $\beta=0.7$ and $M=1$ for MuSiQue; 
    \item \textbf{The Maximal Marginal Relevance (MMR)}~\cite{carbonell1998use}: 
    It greedily selects documents by balancing query relevance and novelty, explicitly penalizing candidates that are similar to previously selected documents.

% \begin{equation}
%     i^*=\operatorname{argmax_{i\notin S}}[(1-\gamma)r_i - \gamma\operatorname{max}_{j \in S}G_{ij}]
% \end{equation}
    % $i^*=\operatorname{argmax_{i\notin S}}[(1-\gamma)r_i - \gamma\operatorname{max}_{j \in S}G_{ij}]$, 
    % where $r_i$ is the standardized score between query and document $i$, and $S$ is the selected document set; $G_{ij}$ is the cosine similarity between document $i$ and $j$; $\gamma$ is the weight to balance query-relevance and information-novelty. 
    % In our experiment, $\gamma=0.35, 0.0, 0.2$ for HotpotQA, 2WikiMultiHopQA, and MuSiQue, respectively. 
    \item \textbf{Dartboard}~\cite{pickett2024better}: 
    It selects documents by maximizing relevant information gain, encouraging the selected context to contain information that is both useful for the query and diverse from previously selected documents. 
    \item \textbf{Adaptive-K}~\cite{taguchi2025efficient}: 
    It heuristically selects $D$ by stopping where query-document similarity has the largest drop between two consecutive ranked documents. 
    Since $D$ is determined by the method rather than fixed externally, we report the mean selected $D$ and its corresponding recall as a single point in the plots. 
\end{itemize}
To further analyze the role of the relevance weight in Eq.~\ref{eq:coverage}, we conduct ablation studies with three variants:
\begin{itemize}
    \item \textbf{w/o query \& anchor relevance}: 
    removes the relevance weight $\rho_i$ from the coverage objective:     
    \vspace{-1mm}
    \begin{equation}
        F(S) = \sum_v q(v) \max_{d_i \in S}\left[ d_i(v) \right].
    \end{equation}
    \vspace{-4mm}
    \item \textbf{w/o query relevance}: 
    uses only soft-anchor relevance by setting $b_j=0$ in Eq.~\ref{eq:anchor_refine}.
    \item \textbf{w/o anchor relevance}: 
    uses only direct query relevance by setting $\bar{a}_j=0$ in Eq.~\ref{eq:anchor_refine}.
\end{itemize}

\subsection{Evidence Frontloading Improves Recall} 
\label{sec:frontload_improve}

\begin{figure*}
    \centering
    \includegraphics[width=\linewidth]{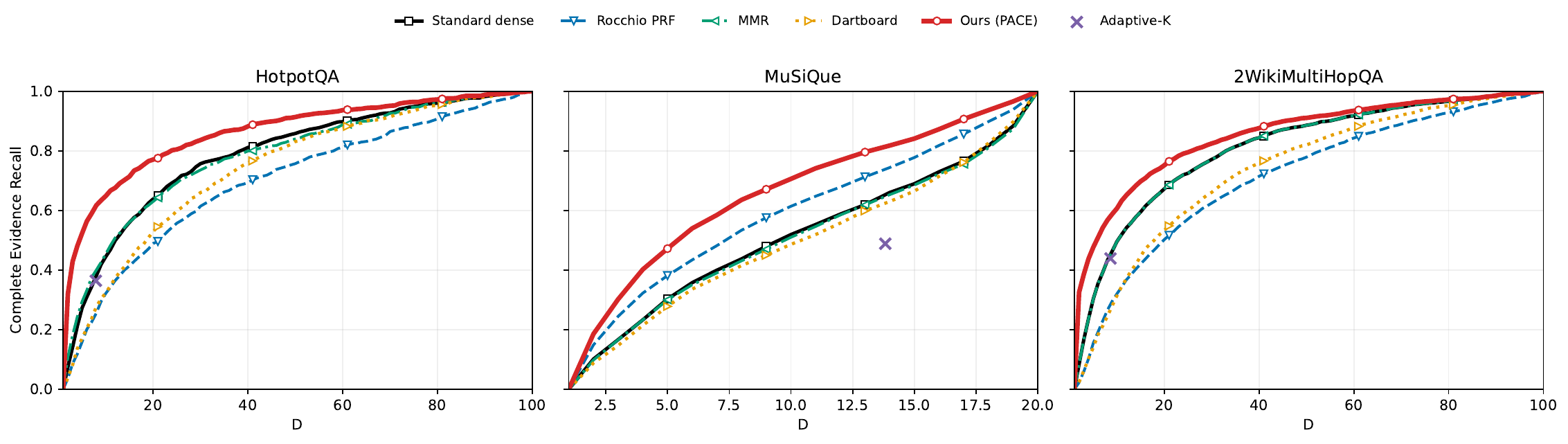}
    \vspace{-8.1mm}
    \caption{
    Complete evidence recall@D across three multi-hop QA datasets. \textsf{PACE} tends to cover all required evidence with smaller reranking budgets than training-free baselines.}
    \label{fig:complete_recall}
    \vspace{-5mm}
\end{figure*}

\begin{figure*}
    \centering
    \includegraphics[width=\linewidth]{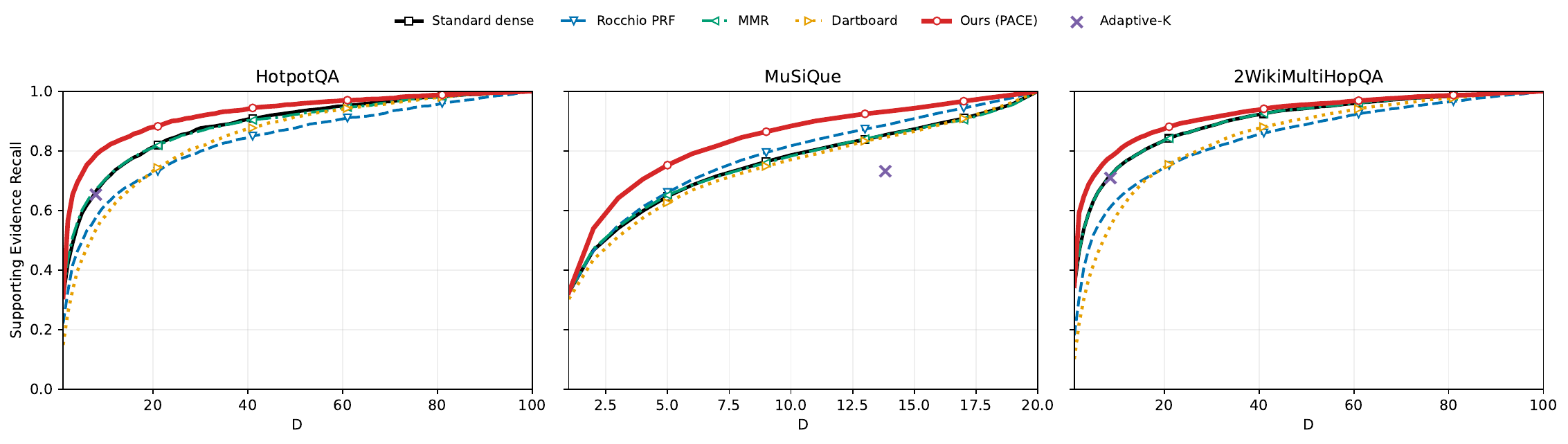}
    \vspace{-8.1mm}
    \caption{Supporting evidence recall@D across three multi-hop QA datasets. 
    \textsf{PACE} retrieves a larger fraction of supporting evidence with small reranking budgets. }
    \label{fig:support_recall}
    \vspace{-5mm}
\end{figure*}

\begin{figure*}
    \centering
    \includegraphics[width=\linewidth]{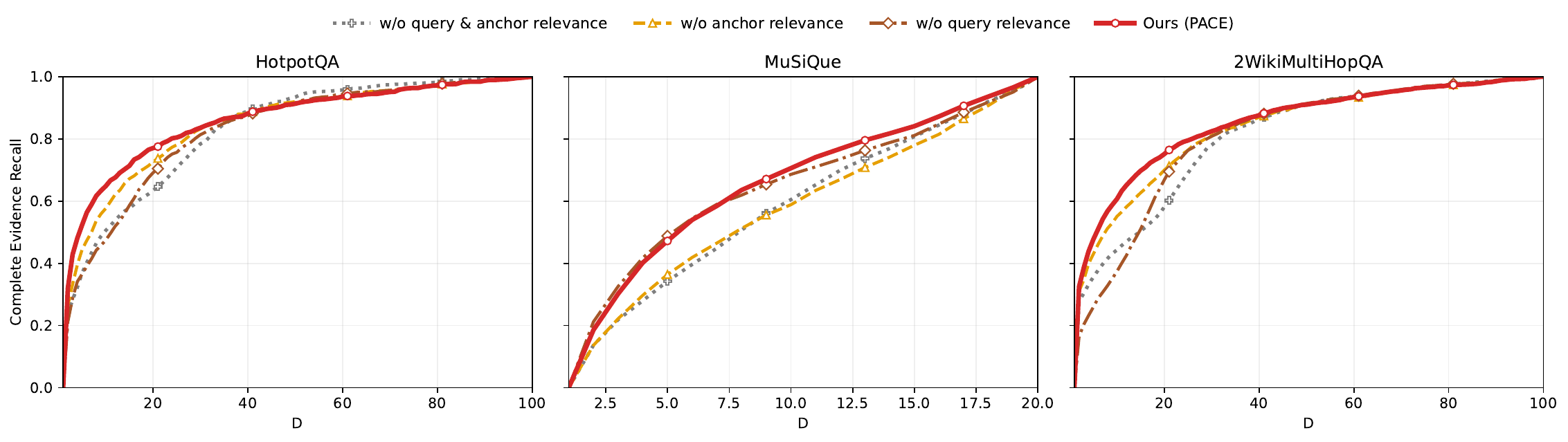}
    \vspace{-8.2mm}
    \caption{
    Ablation study of the marginal evidence coverage objective in Eq.~\ref{eq:coverage}. 
    Removing query relevance, anchor relevance, or both weakens complete evidence recall, showing that direct query relevance and soft-anchor document dependencies jointly contribute to effective evidence frontloading.}
    \label{fig:ablation_complete_recall}
    \vspace{-5mm}
\end{figure*}

We first evaluate the effectivenes of \textsf{PACE} evidence frontloading under a fixed budget of $D=100$, focusing on whether it can move supporting evidence into top-ranked candidates. 
Figure~\ref{fig:complete_recall} and~\ref{fig:support_recall} report evidence recall at different reranking budgets. 
Across all three datasets, \textsf{PACE} consistently improves both complete evidence recall@D and supporting evidence recall@D over training-free baselines, especially at small $D$, which is the key regime for reducing reranker workload. 
For example, on HotpotQA, \textsf{PACE} at $D=20$ achieves comparable evidence recall to the best baseline at $D=40$, using only half of the reranking budget. 

The baselines show that heuristic diversity does not necessarily improve evidence coverage. 
MMR and Dartboard often underperform the original dense ranking because they penalize similar documents, while multi-hop supporting documents can be semantically related and jointly necessary for answering. 
PRF can benefit from expanding the query with top-ranked documents, but it is sensitive to noisy initial results and may suffer from query drift. 
Adaptive-K is shown as a single point because it selects $D$ automatically; its recall is bounded by the original dense ranking at the selected position since it does not reorder documents. 
These results show that \textsf{PACE} \textit{moves necessary evidence into top-ranked candidates, rather than only improving recall at large budgets.}

Figure~\ref{fig:ablation_complete_recall} further analyzes the relevance weight $\rho$ in the marginal evidence coverage objective in Eq.~\ref{eq:coverage}. 
Removing both query and anchor relevance usually leads to the worst performance, showing that pure coverage without relevance guidance cannot reliably identify useful evidence. 
Using only query relevance prioritizes documents directly related to the question, but can miss evidence that is weakly related to the query but instead connected to its one-hop evidence. 
Using only anchor relevance can recover such evidence, but may also promote irrelevant documents when the anchors are noisy. 
As a result, using either source alone is generally suboptimal, especially at small $D$, where \textsf{PACE} consistently achieves higher recall. 
MuSiQue shows a slightly different pattern, where anchor relevance is relatively more effective than query relevance. 
This is likely because MuSiQue uses a closed-context setting with at most 20 candidates, which reduces noise in anchor estimation, and its questions are constructed by composing single-hop questions through bridge entities. 
This also explains why PRF performs relatively better on MuSiQue than on the open-domain datasets, since its pseudo-relevance feedback is less affected by noisy top-ranked documents in the closed-context candidate set. 
Overall, combining query and anchor relevance gives the strongest and most stable evidence frontloading across datasets.

\subsection{\textsf{PACE} Improves Efficiency while Preserving Recall}
\label{sec:pace_improve}

\begin{figure*}
    \centering
    \includegraphics[width=\linewidth]{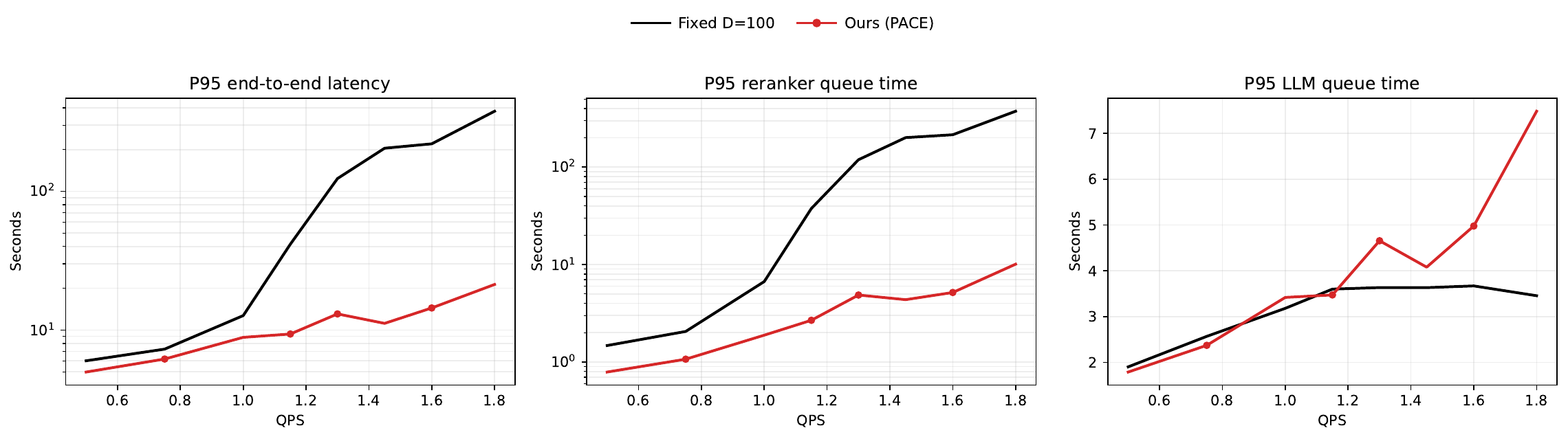}
    \vspace{-7mm}
    \caption{
    Latency comparison under ranking-heavy workloads. 
    \textsf{PACE} substantially reduces p95 end-to-end latency by lowering reranker queueing time with a smaller adaptive reranking budget. 
    Although LLM queueing time may increase as more requests pass through the reranker, but the dominant upstream bottleneck is greatly relieved. The fixed baseline uses $D=100$ throughout.}
    \label{fig:online_latency}
    \vspace{-3mm}
\end{figure*}

\begin{figure*}
    \centering
    \includegraphics[width=\linewidth]{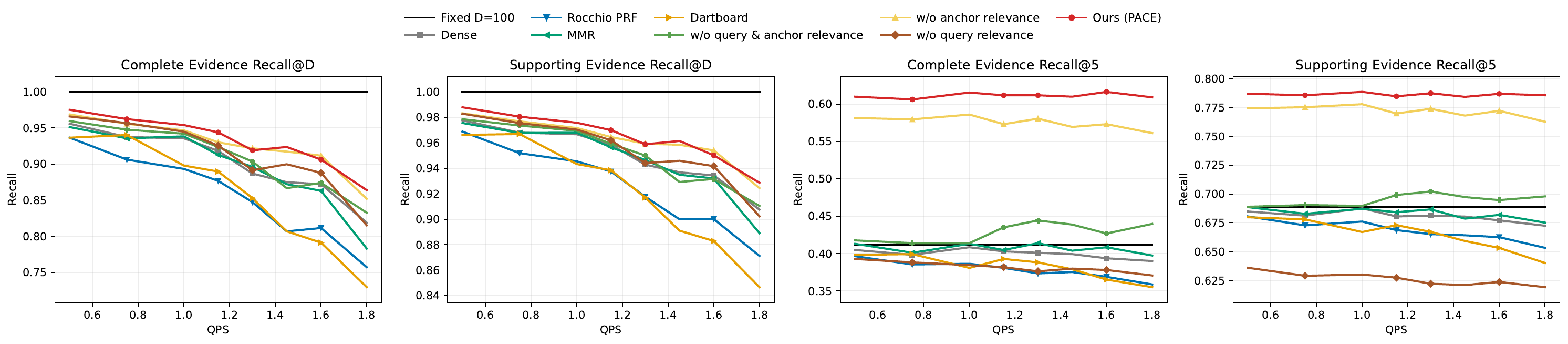}
    \vspace{-7mm}
    \caption{
    Evidence recall under pressure-adaptive budgeting on HotpotQA. 
    The first two subfigures report recall@D before reranking, and the last two report recall@5 after reranking, where $D$ is dynamically selected and $K=5$ documents are sent to the LLM. 
    \textit{Less can be more when the top-ranked candidates are evidence-dense}: as QPS increases, \textsf{PACE} reduces the reranking budget but maintains higher complete and supporting evidence recall than baselines and ablation variants. 
    % The fixed-$D$ baseline keeps $D=100$ for all QPS values. 
    }
    \label{fig:online_recall}
    \vspace{-3mm}
\end{figure*}

\begin{figure}
    \centering
    \includegraphics[width=\linewidth]{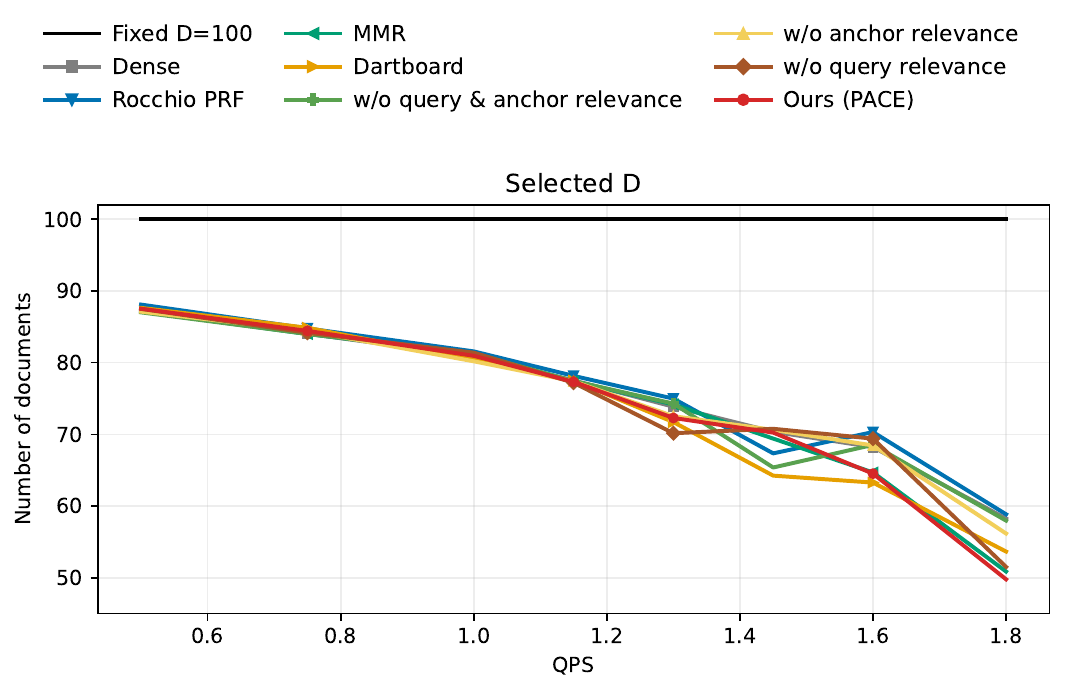}
    \vspace{-8mm}
    \caption{
    Average selected reranking budget under increasing QPS. 
    \textsf{PACE} reduces $D$ as serving pressure grows, but maintains higher evidence recall than other methods at comparable selected budgets.}
    \label{fig:online_select_D}
    \vspace{-7mm}
\end{figure}

The static offline results in Section~\ref{sec:frontload_improve} show that evidence frontloading improves recall across different cutoffs $D$. 
We now evaluate whether \textsf{PACE} can translate this advantage into end-to-end latency gains under online serving, while preserving evidence recall after dynamically reducing $D$. 
We conduct online serving experiments on HotpotQA and report latency in Figure~\ref{fig:online_latency} and recall in Figures~\ref{fig:online_recall}.

\vspace{1mm}
\noindent \textit{Pressure-adaptive budget selection improves system latency.} 
Figure~\ref{fig:online_latency} compares \textsf{PACE} with a fixed $D=100$ reranking budget under ranking-heavy workloads. 
As QPS increases, the fixed-$D=100$ baseline quickly accumulates reranker queueing delay, causing p95 end-to-end latency to grow sharply. 
In contrast, \textsf{PACE} dynamically reduces the reranking budget according to system pressure, which substantially lowers reranker queueing time and stabilizes end-to-end latency. 
Although LLM queue time can increase slightly because more requests pass through the reranker, the dominant upstream bottleneck is relieved, leading to much lower overall latency.

\vspace{1mm}
\noindent \textit{\textbf{\textsf{PACE}} preserves recall under smaller adaptive budget.} 
To fairly compare methods in the online serving setting, we apply the same pressure-adaptive budgeting policy to all reordering methods. 
As a result, all adaptive methods operate with similar latency and comparable selected reranking budgets, since reordering a small candidate set is negligible compared with total serving latency; for example, \textsf{PACE} selection accounts for less than 0.007\% of p95 end-to-end latency. 
This allows us to focus on how evidence recall changes under increasing serving load. 
Figure~\ref{fig:online_select_D} shows that as QPS increases, \textsf{PACE} gradually reduces $D$, similar to other adaptive methods. 
At QPS$=1.8$, most adaptive methods select a budget around $50$--$60$ documents to relieve reranker pressure.

Figure~\ref{fig:online_recall} reports evidence recall at two stages: recall@D before reranking, and recall@K after reranking, where we use the common setting $K=5$~\cite{chirkova2025provence}. 
Before reranking, the fixed $D=100$ baseline maintains recall@D $1.0$ because no candidates are dropped. 
All adaptive methods show lower recall@D as QPS increases because their selected $D$ decreases. 
However, \textsf{PACE} maintains higher complete and supporting evidence recall than baselines and ablation variants at comparable budgets. 
This shows that its latency reduction does not come from simply dropping candidates, but from combining adaptive budgeting to an evidence-dense  candidate ranking.

After reranking, \textsf{PACE} shows that \textit{less can be more when the top-ranked candidates are evidence-dense}: it achieves the highest recall@5 among the documents sent to the LLM, even surpassing the fixed $D=100$, and remains stable as QPS increases. 
This suggests that a larger reranking budget does not always improve final top-$K$ evidence recall, since additional candidates may introduce noise. 
By frontloading useful evidence before reranking and reducing $D$ under pressure, \textsf{PACE} provides the reranker with a smaller but more evidence-dense candidate set, making useful evidence more likely to appear in the final context. 
For example, at QPS$=1.8$, \textsf{PACE} uses nearly half the reranking budget of fixed $D=100$, but achieves about 20\% higher recall@5. 
In contrast, most baselines fall below fixed $D=100$ after reranking, and their recall generally decreases as QPS increases because the adaptive budget becomes smaller. 
We further validate this effect under different fixed reranking budgets on two datasets in Table~\ref{tab:less_be_more}: \textsf{PACE} achieves higher final recall at $K=5$ than the full-budget baseline even with substantially smaller $D$, showing that its gains come from evidence frontloading rather than simply preserving more candidates.

\begin{table}[!t]
\centering
\small
\setlength{\tabcolsep}{6pt}
\caption{\textit{Less can be more}. Complete evidence recall under fixed reranking budgets. Fixed $D=100$ denotes the full-budget baseline, and $R@K$ with $K=5$ measures final evidence recall after reranking. Values in parentheses indicate absolute percentage-point changes in $R@K$ over the full-budget baseline on the same dataset.}
\label{tab:less_be_more}
\vspace{-1mm}
\begin{tabular}{llccc}
\toprule
Dataset & Method & $D$ & Comp. R@$D$ & Comp. R@$K$ \\
\midrule
\multirow{7}{*}{HotpotQA}
& Standard Dense & 100 & 100 & 41.13 \\
\cmidrule(lr){2-5}
& \multirow{3}{*}{Dartboard} 
  & 80 & 95.31     & \textcolor{negred}{{40.48 {\scriptsize $(-0.65)$}}}  \\
& & 50 & 83.17 &  \textcolor{negred}{{38.18 {\scriptsize $(-2.95)$}}} \\
& & 20 & 52.99 & \textcolor{negred}{{33.22 {\scriptsize $(-7.91)$}}} \\
% \cmidrule(lr){2-5}
% & \multirow{3}{*}{Dartboard} 
%   & 30 & -- & -- \\
% & & 50 & -- & -- \\
% & & 70 & -- & -- \\
\cmidrule(lr){2-5}
& \multirow{3}{*}{\textsf{PACE}} 
  & 80 & 97.25 & \textcolor{posgreen}{\textbf{60.81 {\scriptsize $(+19.68)$}}} \\
& & 50 & 91.36 & \textcolor{posgreen}{\textbf{62.38 {\scriptsize $(+21.25)$}}} \\
& & 20 & 77.10 & \textcolor{posgreen}{\textbf{60.26 {\scriptsize $(+19.13)$}}} \\
\midrule
\multirow{7}{*}{2Wiki}
& Standard Dense & 100 & 100 & 48.76 \\
\cmidrule(lr){2-5}
& \multirow{3}{*}{Dartboard} 
  & 80 & 95.15 & \textcolor{negred}{{48.10 {\scriptsize $(-0.66)$}}} \\
& & 50 & 81.97 & \textcolor{negred}{{47.40 {\scriptsize $(-1.36)$}}} \\
& & 20 & 53.41 & \textcolor{negred}{{39.96 {\scriptsize $(-8.80)$}}} \\
% \cmidrule(lr){2-5}
% & \multirow{3}{*}{Dartboard} 
%   & 30 & -- & -- \\
% & & 50 & -- & -- \\
% & & 70 & -- & -- \\
\cmidrule(lr){2-5}
& \multirow{3}{*}{\textsf{PACE}} 
  & 80 & 97.24 & \textcolor{posgreen}{\textbf{54.22} {\scriptsize $(+5.46)$}} \\
& & 50 & 91.06 & \textcolor{posgreen}{\textbf{54.39} {\scriptsize $(+5.63)$}} \\
& & 20 & 75.19 & \textcolor{posgreen}{\textbf{53.94} {\scriptsize $(+5.18)$}}\\
\bottomrule
\end{tabular}
\vspace{-4mm}
\end{table}

%% file: Sections/Related_Work.tex
\vspace{-1mm}
\section{Related Work}

\noindent \textbf{Relieving generation-heavy bottlenecks.}
Existing work on efficient RAG often reduces the cost of downstream LLM generation. 
Serving systems improve throughput and latency through optimized batching, scheduling, KV-cache management, and prefill execution~\cite{yu2022orca,kwon2023efficient,zheng2024sglang,agrawal2024taming}. 
METIS~\cite{ray2025metis} adapts query-level configurations, such as retrieved chunks and synthesis methods, to balance response quality and latency. 
Context compression methods, such as RECOMP~\cite{xu2023recomp}, LLMLingua~\cite{jiang2023llmlingua}, and Provence~\cite{chirkova2025provence}, shorten retrieved contexts by removing or summarizing less useful tokens. 
These methods reduce generation-side cost, but they do not directly address shifting bottlenecks between reranking and generation or relieve reranker-side pressure while preserving evidence recall.

% reranker can be expensive~\cite{dejean2026efficient,yu2024rankrag}, especially when the recalled documents are lot, which can be bottleneck~\cite{an2025hyperrag}. 
% Early-exit methods reduce reranking cost by making decisions at intermediate Transformer layers, either using auxiliary classifiers~\cite{xin2020early} or layer-wise query-document similarity estimates before completing the full forward pass~\cite{busolin2025efficient}. 
% RRK~\cite{dejean2026efficient} accelerates listwise LLM reranking by replacing full document texts with learned fixed-size multi-token representations, substantially reducing the sequence length processed by the reranker. 
% Adaptive-k~\cite{taguchi2025efficient} heuristically selects $D$ by stopping where query-document similarity has the largest drop between two consecutive ranked documents. Thus it can relieve reranker pressure at some degree indirectly. 
% In our work, we do not change the reranker model, instead, we change the input of reranker, the order and the number of documents to both improve the recall as well as the latency. 

\vspace{0.5mm}
\noindent \textbf{Relieving ranking-heavy bottlenecks.}
Rerankers improve retrieval quality but become expensive when many query-document pairs must be scored~\cite{dejean2026efficient,yu2024rankrag,an2025hyperrag}. 
Prior work mainly reduces this cost by accelerating the reranker itself. 
Early-exit methods make decisions at intermediate Transformer layers using auxiliary classifiers~\cite{xin2020early} or layer-wise query-document similarity estimates~\cite{busolin2025efficient}. 
% RRK~\cite{dejean2026efficient} accelerates listwise LLM reranking by replacing full document texts with learned fixed-size representations. 
Adaptive-K~\cite{taguchi2025efficient} selects the budget at the largest drop in query-document similarity. 
These methods reduce reranking computation or choose a heuristic budget point, but they do not explicitly adapt the reranking budget to real-time reranker--LLM pressure. 
\textsf{PACE} is complementary: it keeps the reranker unchanged, but controls which candidates enter reranking and how many are processed under online serving pressure.

% \noindent \textbf{Improving evidence recall. }
% (1) Multi-hop evidence retrieval: 
% MDR~\cite{xiong2020answering} extends dense retrieval to multi-hop QA by recursively conditioning each retrieval step on evidence obtained in previous hops. 
% Baleen~\cite{khattab2021baleen} improves multi-hop retrieval through condensed intermediate evidence, focused late interaction, and latent hop ordering to control search-space growth and model complex information needs.
% However, these methods need to train a new retriever, contrast, our proposed method is training-free. 
% IRCoT~\cite{trivedi2023interleaving} interleaves chain-of-thought reasoning with retrieval, allowing intermediate reasoning steps to guide subsequent evidence acquisition and vice versa. Though it is training-free, the interleave of LLM reasoning and retrieval obviously cannot relieve but worsen the RAG bottleneck. 
% (2) diversity-aware retrieval: 
% MMR~\cite{carbonell1998use} greedily balances query relevance against redundancy with previously selected documents to produce diversified retrieval results. 
% xQuAD~\cite{santos2010explicit} explicitly models a query as multiple sub-queries and diversifies retrieval by favoring documents that cover relevant but underrepresented query aspects. 
% Dartboard~\cite{pickett2024better} replaces explicit relevance–diversity trade-offs with an information-theoretic objective that maximizes relevant information gain, causing diversity to emerge implicitly during retrieval. 

\vspace{0.5mm}
\noindent \textbf{Improving evidence recall.}
Evidence recall is often improved through multi-hop retrieval or diversity-aware ranking. 
Multi-hop retrievers acquire supporting evidence through iterative retrieval or reasoning-guided search: 
MDR~\cite{xiong2020answering} recursively conditions later retrieval steps on previously retrieved evidence, Baleen~\cite{khattab2021baleen} uses condensed intermediate evidence to guide subsequent hops, and IRCoT~\cite{trivedi2023interleaving} interleaves LLM reasoning with retrieval. 
These methods target evidence acquisition, but require specialized retriever training or additional LLM-retrieval iterations, which can increase serving cost. 
Diversity-aware methods, which improves recall by selecting non-redundant documents from a candidate set.
MMR~\cite{carbonell1998use} balances query relevance with novelty, xQuAD~\cite{santos2010explicit} promotes coverage of different query aspects, and Dartboard~\cite{pickett2024better} maximizes relevant information gain for diverse RAG retrieval.
These methods are closer to our setting because they reorder or select from existing candidates. 
However, they typically rely on coarse document-level distances or manually designed trade-offs, while multi-hop QA may require preserving documents that are similar but jointly necessary. 
\textsf{PACE} differs by \textit{measuring complementarity over fine-grained semantic dimensions}, \textit{supporting both diversity and evidence chaining without additional training or tuned parameters}, and \textit{providing monotone submodular objective with a $(1-1/e)$ approximation guarantee}.